\documentclass[11pt]{article}
\usepackage[margin=1in]{geometry}
\usepackage[utf8]{inputenc} 
\usepackage[T1]{fontenc}    
\usepackage{hyperref}       
\usepackage{url}            
\usepackage{booktabs}       
\usepackage{amsfonts}       
\usepackage{nicefrac}       
\usepackage{microtype}      
\usepackage{amsmath,amssymb,amsthm,bm}
\usepackage{algorithm,algorithmic}
\usepackage{adjustbox}
\usepackage{graphicx}
\usepackage{latexsym}
\usepackage{listings}
\usepackage{xcolor}
\usepackage{placeins}
\usepackage{multirow}
\usepackage{tabu}
\usepackage{authblk}
\usepackage[resetlabels]{multibib}

\newcites{ltex}{References} 

\usepackage{mathtools}

\usepackage{hyperref}

\newcommand{\norm}[1]{\left\lVert#1\right\rVert}
\newcommand{\innerproduct}[2]{\left\langle#1, #2\right\rangle}
\newcommand{\mc}[1]{\mathcal{#1}}
\newcommand{\mb}[1]{\mathbb{#1}}
\newcommand{\mr}[1]{\mathrm{#1}}
\newcommand\numberthis{\addtocounter{equation}{1}\tag{\theequation}}
\newcommand*{\Scale}[2][4]{\scalebox{#1}{$#2$}}%

\newcommand{\w}{\bm{w}}
\newcommand{\x}{\bm{x}}
\newcommand{\y}{\bm{y}}
\newcommand{\z}{\bm{z}}
\newcommand{\s}{\bm{s}}
\newcommand{\e}{\bm{e}}
\newcommand{\g}{\bm{g}}
\newcommand{\q}{\bm{q}}
\newcommand{\bu}{\bm{u}}
\newcommand{\br}{\bm{r}}
\newcommand{\bv}{\bm{v}}
\newcommand{\bmu}{\bm{\mu}}
\newcommand{\bnu}{\bm{\nu}}
\newcommand{\btheta}{\bm{\theta}}
\newcommand{\mE}{\mb{E}}

\newcommand*\samethanks[1][\value{footnote}]{\footnotemark[#1]}

\DeclareMathOperator{\argmin}{\arg\!\min}

\DeclareMathOperator{\pr}{\textbf{Pr}}

\newtheorem{lem}{Lemma}
\newtheorem{cor}{Corollary}
\newtheorem{thr}{Theorem}
\newtheorem{defn}{Definition}
\newtheorem{prop}{Proposition}

\title{Collect at Once, Use Effectively:\\
Making Non-interactive Locally Private Learning Possible}

\makeatletter
\newcommand{\thline}{%
    \noalign {\ifnum 0=`}\fi \hrule height 1pt
    \futurelet \reserved@a \@xhline
}
\newcolumntype{"}{@{\hskip\tabcolsep\vrule width 1pt\hskip\tabcolsep}}
\makeatother

\author{Kai Zheng\thanks{These two authors contributed equally}~\thanks{zhengk92@pku.edu.cn} }
\author{Wenlong Mou\samethanks[1]~\thanks{mouwenlong@pku.edu.cn}}
\author{Liwei Wang\thanks{wanglw@cis.pku.edu.cn}}
\affil{Key Laboratory of Machine Perception, MOE,\\School of EECS, Peking University}

\begin{document} 
\maketitle

\begin{abstract} 

Non-interactive Local Differential Privacy (LDP) requires data analysts to collect data from users through noisy channel at once. In this paper, we extend the frontiers of Non-interactive LDP learning and estimation from several aspects. For learning with smooth generalized linear losses, we propose an approximate stochastic gradient oracle estimated from non-interactive LDP channel using Chebyshev expansion, which is combined with inexact gradient methods to obtain an efficient algorithm with quasi-polynomial sample complexity bound. For the high-dimensional world, we discover that under $\ell_2$-norm assumption on data points, high-dimensional sparse linear regression and mean estimation can be achieved with logarithmic dependence on dimension, using random projection and approximate recovery.  We also extend our methods to Kernel Ridge Regression. Our work is the first one that makes learning and estimation possible for a broad range of learning tasks under non-interactive LDP model.
\end{abstract} 

\section{Introduction}
\label{introduction}

Data privacy has become an increasingly important issue in the age of data science. Differential Privacy (DP), proposed in 2006 by Dwork et al.,\cite{dwork2006calibrating}, provide a solid foundation and rigorous standard for private data analysis. Since then, there has been extensive literature studying the fundamental trade-offs between differential privacy and accuracy for query answering~\cite{Hardt2010A,Hardt2012A,Thaler2012Faster,Wang2016Differentially}, machine learning~\cite{Chaudhuri2008Privacy,chaudhuri2011differentially,rubinstein2012learning,wangyuxiang2015}, and statistical inference~\cite{lei2011differentially,Smith2011Privacy}. For more details on DP results, please refer to the excellent monograph written by Dwork and Roth \cite{dwork2014algorithmic}. Intuitively, a DP algorithm uses randomized response to defend against adversary, so that change of one of data points could not be detected.

Despite the prevailing success of this notion in academia, its applicability in data science practice could be limited. For example, if data analysts just promise to follow the differential privacy constraints, user will not feel their privacy are preserved. The promise could not be validated; the mechanisms are complicated; and even worse: users do not trust the data collector at all. Unfortunately, most of differential privacy algorithms are based on adding noise calibrated to stability of loss function, which essentially requires access to original data. 

Borrowing ideas from classical wisdom on collecting sensitive survey data~\cite{warner1965randomized}, Local Differential Privacy (LDP)~\cite{Kasiviswanathan2010What,duchi2013locala} was proposed as a stronger notion of privacy to resolve this problem. LDP requires each of data points to be passed through a noisy channel during collection. This channel will ensure one can hardly tell anything about the user based on what he have sent. The practical advantage of LDP is obvious: users will be comfortable sending their sensitive information through noisy channels, which are transparent and reliable; additionally, users can choose their own privacy parameters, making it possible to associate with economic value. Therefore, this line of research has attracted lots of attention~\cite{duchi2013localb,duchi2013locala,kairouz2014extremal,bassily2015local,kairouz2016discrete}.

Despite the analogy in definition, the way in which LDP achieves accurate results are fundamentally different from classical DP. Essentially, the information collected from each user is almost completely noisy, from which one needs to obtain accurate results. The only way to do that is to make the independently distributed noise cancel out with each other in some sense. With sand being washed away by waves, golds begin to appear.

Two local privacy notions have been discussed in existing literature: the interactive model allows the algorithm to collect data sequentially, and decide what to ask based on information from previously asked users. The non-interactive model, on the contrary, requires all data to be collected at once, with no interactive queries allowed. Apparently the non-interactive model is strictly stronger, and prohibition on interactive queries rules out most of SGD-type approaches, making the problem significantly harder. However, non-interactive LDP is more useful in real-world applications, as opportunities of interactive queries may not be available in most settings.

In existing literature, learning and inference under interactive and non-interactive LDP therefore are exhibiting different appearances. In the interactive world, LDP is promised with connection to Statistical Query (SQ) model~\cite{kearns1998efficient}, from its very beginning~\cite{kasiviswanathan2011can}. SQ algorithms for a wide range of convex ERM problems were proposed by~\cite{feldman2017statistical}, implying good risk bounds for LDP. \cite{duchi2013locala} established matching upper and lower bounds for convex risk minimization problems. On the other hand, very few has been done in the non-interactive setting. Existing works primarily focus on basic estimation problems such as means and discrete densities~\cite{duchi2013localb,duchi2016minimax,bassily2015local}, or some function calculations \cite{kairouz2015secure}. Most of important modern learning and inference tasks, including estimation in linear models and convex ERM, are still poorly understood in non-interactive local DP settings. 

For the high-dimensional world, where $d\gg n$ while some low-complexity constraints are imposed, we may hope the error induced by privacy constraints to be logarithmically dependent upon $d$. In classical differential privacy literature, this has been be addressed using different techniques, guarantee error bounds logarithmically dependent on dimension~\cite{talwar2015nearly,smith2013differentially}. However, lower bounds have been shown in local privacy model even for high-dimensional $1$-sparse mean estimation, ruling out any good guarantees~\cite{duchi2016minimax}. The lower bound result illustrates fundamental difficulties of local differential privacy. But if we still want to do high-dimensional learning under local privacy, are there additional assumptions that helps?

Therefore, the starting point of this work lies on making learning possible under the non-interactive LDP setting, which is the hardest yet the most useful. We initiate the first attempt towards a broad range of learning tasks beyond simple distribution estimation. In particular, we investigate two important classes of problems under non-interactive LDP: (1) High-dimensional sparse linear regression and mean estimation; (2) Generalized linear models. Our focus is to design corresponding mechanisms and study their convergence rates with respect to the number and dimension of data. One can also consider optimal mechanisms in terms of privacy parameters like \cite{geng2014optimal}, which is of independent interests.



{\bf Our Contributions:} In this paper, we propose several efficient algorithms for learning and estimation problems under non-interactive LDP model, with good theoretical guarantees. In the following we summarize our contributions.

{\bf (1) High Dimensional Estimation:} One of exciting findings in this paper is about local privacy for high-dimensional data. Roughly speaking, convergence rate with logarithmic dependence on the dimension can be attained under LDP, if we assume data points are $\ell_2$ bounded. This is in sharp contrast with information-theoretic lower bounds for 1-sparse mean estimation for $\ell_\infty$ bounded data~\cite{duchi2016minimax}. Valid algorithms are presented for both sparse mean estimation and sparse linear regression, respectively. Intuitively, non-interactivity doesn't bring about additional difficulties, since the loss functions are quadratic forms. However, if we directly add noise to each of data points and send it to the server, the aggregated noise will lead to linear dependence on the dimension. Thus we adopt the random projection technique, and send the noisy version of projected data to the server. Based on the aggregated information, we can approximately recover the optimal solution via linear inverse problem.

{\bf (2) Learning Smooth Generalized Linear Models:} 
Generalized linear problems which has additional smooth properties (we call the loss with respect to it as smooth generalized linear loss (SGLL), see rigorous definition in section 2) include many common loss functions, such as logistic loss, square loss, etc. Optimizing such losses are intuitively much more difficult in non-interactive LDP model, as the loss can be an arbitrary function $\w^T\x$. This even makes it difficult for us to obtain an unbiased estimator for objective function, or its gradient. As a result, when we aggregate the loss of noisy data together, it is even hard to ensure it converge to the population loss. Approximation theory techniques are introduced to tackle this problem. In particular, we use polynomials of $\w^T\x$ to approximate nonlinear coefficients of gradients. Chebyshev bases, instead of Taylor series, are used to get faster convergence within an arbitrary domain. Then we are able to build inexact stochastic gradient oracles to arbitrarily specified accuracy. SIGM algorithm in \cite{dvurechensky2016stochastic} is exploited to find the minimizer with inexact gradients.

{\bf Other Related Work:} 
Local privacy dates back to~\cite{warner1965randomized}, who uses random responses to protect privacy in surveys. In recent LDP literature, both~\cite{duchi2013localb} and~\cite{kairouz2016discrete} studied density estimation methods and their theoretical behaviors in LDP model. Rather than statistical setting in above two work,~\cite{bassily2015local} considered how to produce frequent items and corresponding frequencies of a dataset in local model. Besides,~\cite{kairouz2014extremal} investigated optimality of LDP mechanisms based on information theoretical measures for statistical discrimination.

Approximation techniques are commonly used in DP literature.~\cite{Thaler2012Faster} employed polynomials for marginal queries.~\cite{Wang2016Differentially} leveraged trigonometric polynomials to answer smooth queries. \cite{zhang2012functional} also used polynomial approximations and get basic convergence results in standard DP model. Besides, the random projection and recovery has also been used in DP learning~\cite{kasiviswanathan2016efficient} and local DP histogram estimation~\cite{bassily2015local}.

In standard DP model, both high-dimensional sparse estimation and generalized linear model have been intensively studied. \cite{kifer2012private} and \cite{smith2013differentially} considered the convergence of private LASSO estimator under RSC and incoherence assumptions. \cite{talwar2015nearly} considered constrained ERM of sparse linear regression, and obtained $\tilde{O}(\log d / n^{2/3})$ rate using private Frank-Wolfe. Above results assume $\ell_{\infty}$-bounded data. By stronger assumption of $\ell_2$ bounded data, \cite{kasiviswanathan2016efficient} gave a general framework for high dimensional empirical risk minimization (ERM) problem. There are several works to estimate generalized linear model under DP, with a particular emphasis on logistic regression. Objective and output perturbation are used to get low excess risks~\cite{Chaudhuri2008Privacy,chaudhuri2011differentially}. Both \cite{bassily2014private} and \cite{zhang2017efficient} considered concrete private algorithms to solve ERM. None of these existing results extends directly to non-interactive LDP setting.


\section{Preliminaries}
{\bf Some notations:} $[p] = \{1,2,\cdots, p\}$. Vectors are written in bold symbol, such as $\bm{x}, \w$. $x$ represents univariate number, which has no relation with $\x$. For a vector $\bm{x} = [x_1, x_2, \cdots, x_d]^T$, $\bm{x}^k$ represents the power of each element. $B_2(r) = \{\x|\norm{\x}_2 \leqslant r\}$. Denote $\mb{S}^{+}$ as the semipositive matrix space, $\mathrm{Proj}_{\mb{S}^{+}}(\cdot)$ means projecting a matrix to $\mb{S}^{+}$ in terms of Frobenius norm (i.e. eliminate all negative eigenvalues). For an univariate function $f(x), f^{(k)}(x)$ represents its $k$-th derivative, and define $\norm{f^{(k)}}_T :=  \int_{-1}^1\frac{|f^{(k+1)}(x)|}{\sqrt{1-x^2}}\mr{d}x$. For the reason of limited space, all omitted proof can be found in the supplementary.

\subsection{Local Differential Privacy}
Here we adopt the LDP definition given in \cite{bassily2015local}.
\begin{defn}
  A mechanism $\mc{Q}: \mc{V} \rightarrow \mc{Z}$ is said to be $(\epsilon, \delta)$-local differential private or $(\epsilon, \delta)$-LDP, if for any $\bv, \bv' \in \mc{V}$, and any (measurable) subset $S \subset \mc{Z}$, there is 
  \begin{align*}
    \pr[Q(\bv) \in S] \leqslant e^{\epsilon}\pr[Q(\bv') \in S] + \delta
  \end{align*}
\end{defn}
Just the same with basic results in DP \cite{dwork2014algorithmic}, there are corresponding basic results for LDP:
\begin{lem}[Gaussian Mechanism]
    \label{gaulem}
  If $\mc{V} = \{\bv \in \mb{R}^d | \norm{\bv}_2 \leqslant 1\}$, then $\mc{Q}(\bv) = \bv + \e$ is $(\epsilon, \delta)$-LDP, where $\e \in \mb{R}^d$, and $\e \sim \mc{N}(0, \sigma^2 I_d)$, $\sigma = 2\sqrt{2\ln(1.25/\delta)}/\epsilon$.
 \end{lem} 
\begin{lem}[Composition Theorem\footnote{Note one can also use the advanced composition mechanism \cite{kairouz2015composition} with a refined analysis, but the main dependence over $n$ and $d$ will remain nearly the same.}]
    \label{composition}
  Let $\mc{Q}_i: \mc{V} \rightarrow \mc{Z}_i$ be an $(\epsilon_i, \delta_i)$-LDP mechanism for $i \in [k]$. Then if $\mc{Q}_{[k]}: \mc{V} \rightarrow \prod_{i=1}^k \mc{Z}_i$ is defined to be $\mc{Q}_{[k]}(\bv) = (\mc{Q}_1(\bv), \dots, \mc{Q}_k(\bv))$, then $\mc{Q}_{[k]}$ is $(\sum_{i=1}^k\epsilon_i, \sum_{i=1}^k \delta_i)$-LDP. 
\end{lem}
 The following simple mechanism add Gaussian noise to preserve LDP of a vector, which serves as a basic tool in LDP learning and estimation. 
\begin{algorithm}
  \caption{Basic Private Vector mechanism}
  \label{LDPgaussian}
  \begin{algorithmic}[1]
    \REQUIRE A vector $\x \in \mb{R}^d$, privacy parameter $\epsilon, \delta$ for LDP
    \ENSURE Private vector $\z$ 
    \STATE Setting $\sigma = \frac{\sqrt{2\ln(1.25/\delta)}}{\epsilon}$
    \IF {$\norm{\x}_2 > 1$}
    \STATE $\x = \x/\norm{\x}_2$
    \ENDIF
    \STATE $\z \leftarrow \x + \e$, where $\e \sim \mc{N}(0, \sigma^2I_d)$
  \end{algorithmic}
  \end{algorithm}
  \begin{thr}
  Algorithm 1 preserves $(\epsilon,\delta)$-LDP.
  \end{thr}


\section{High Dimensional and Non-parametric Learning via Random Projections}
In this section we consider three learning problems under non-interactive LDP: Mean Estimation and Linear Regression in High-dimensions, as well as Kernel Ridge Regression. Using random projection techniques, we are able to get logarithmic dependence on $d$ in high-dimensional settings, and also to get good guarantees for Kernel version. The first problem is considered in statistical settings, as we need to assume a sparse mean vector. The latter two problems are considered as ERM problems, which can easily be translated to population risk using uniform convergence.
\subsection{High-dimensional Mean Estimation}
In this section, we propose a non-interactive LDP mechanism for high-dimensional sparse mean estimation problem. By assuming $\ell_2$ bounded data points, and $\ell_1$ bounded  population mean, we can get error rates with logarithmic dependence on $d$. Our results are in sharp contrast with the lower bound for $\ell_2$-bounded general mean estimation under standard DP~\cite{bassily2014private}, as well as the lower bound for $\ell_\infty$-bounded 1-sparse mean estimation under local DP~\cite{duchi2016minimax}. It can be easily seen that our method extends to mean estimation problem for arbitrary low-complexity constraint set in high dimensions. We state our results in $\ell_1$ setting to keep the arguments clear. Our problem adopts a statistical estimation setting as follows:

\textbf{$\ell_2$-bounded sparse mean estimation} Suppose there is an unknown distribution $\mathcal{D}$ supported on $\mathcal{B}(0,1)$, with $\Vert E_{\mathcal{D}} (\x)\Vert_1\leq \Lambda$. The $\ell_2$-bounded sparse mean estimation problem requires us to produce an estimator $\hat{\btheta}$ that makes $\Vert \btheta-E_{\mathcal{D}} (\x)\Vert_2$ small with high probability.

\begin{algorithm}[htb]
\caption{LDP $\ell_1$ Constrained Mean Estimation}\label{sparsemean}
\begin{algorithmic}
  \REQUIRE $\x_1,\x_2,\cdots,\x_n\sim i.i.d.\mathcal{D}$
  \ENSURE Estimator $z$
  \STATE Set $p=\lceil\Lambda \epsilon\sqrt{n}\rceil$, and $m=\lceil18\log\frac{1}{\delta}\rceil$ 
  \STATE Sample $G\sim \frac{1}{\sqrt{p}}\mathcal{N}(0,1)^{p\times d}$.
  \FOR{User $i$} 
  \STATE Collect $\y_i=G\x_i+\br_i$, \\
  with $\br_i\sim i.i.d.\mathcal{N}(0,\frac{2\log (1.25/\delta)}{\epsilon^2}I_p)$
  \ENDFOR
  \FOR{$j\in\{1,2,\cdots,m\}$}
   \STATE $S_j=\left\{ 1+\frac{(j-1)n}{m},2+\frac{(j-1)n}{m},\cdots,\frac{jn}{m}\right\}$.
  \STATE Let $\bmu_j=\frac{1}{|S_j|}\sum_{i\in S_j}\y_i$.
  \ENDFOR
  \STATE Let $\mathcal{M}=\{\bmu_1,\bmu_2,\cdots,\bmu_m\}$.
  \FOR{$j\in \{1,2,\cdots,m\}$}
  \STATE Let $r_j=\min\left\{r:|\mathcal{B}_{\ell_1}(\bmu_j,r)\cap\mathcal{M}|\geq \frac{m}{2} \right\}$.
  \ENDFOR
  \STATE Let $j_*=\arg\min_j{r_j}$, and $\bu=\bmu_{j_*}$
  \STATE Solve the following convex program:
  \begin{equation}
  \begin{split}
  \arg\min_{\z} &\Vert \z\Vert_1\\
  s.t. \Vert G\z-\bu\Vert_1\leq &\frac{100p\log (nd/\delta)}{\epsilon}\sqrt{\frac{m}{n}}
  \end{split}
  \end{equation}
  \end{algorithmic}
\end{algorithm}

In Algorithm~\ref{sparsemean} we describe our data collection procedure and estimation algorithm. We are primarily using two techniques: random projection and recovery from low-complexity structures; median-of-mean estimator to boost failure probability. The privacy argument is directly implication of Theorem 1.

Intuitively, adding noise to each entry of mean vector will result in error rate's linear dependence on $d$. Thus we adopt the random projection technique to send a compressed version of data vector through the noisy channel. This locally private estimation procedure can be viewed as a variant of noisy compressed sensing, where $\ell_2$ recovery rate is fundamentally controlled by the Gaussian Mean Width of constraint set~\cite{vershynin2015estimation}. Though the distribution has bounded support, the concentration for mean estimation is dimension-dependent, while dimension-independent Markov Inequalities hold. To tackle this problem, we employ Median-of-Mean estimator to get exponential tails~\cite{hsu2016loss}.

We first give the following bound on the error in projected space. 

\begin{lem}\label{groupconstantprob}
  Let $\x_1,\x_2,\cdots,\x_n\sim i.i.d.\mathcal{D}$ with $\bmu=E_{\mathcal{D}}[\x]$ and $supp(\mathcal{D})\subseteq \mathcal{B}(0,1)$. Let $G$ and $\{\y_i\}_{i=1}^n$ defined in the above procedure. For each of group $S_j$ fixed, we have the following with probability $2/3$:
  \begin{equation}
  \Big\Vert \frac{1}{|S_j|}\sum_{\y_i\in S_j}\y_i-G\bmu\Big\Vert_1\leq O\left(\frac{p\log (nd)}{\epsilon\sqrt{|S_j|}}\right)
  \end{equation}
\end{lem}

The aggregation step in Algorithm~\ref{sparsemean} is a high-dimensional generalization of Median-of-Mean estimator used in heavy-tailed statistics. The tail properties are guaranteed in the following lemma:
\begin{lem}[Proposition 9 in~\cite{hsu2016loss}]\label{highmean}
Suppose in metric space $\mathcal{X}$, a set of points $\mathcal{M}=[\btheta_1,\btheta_2,\cdots,\btheta_m]\sim i.i.d.\mathcal{D}$, with $Pr[d_{\mathcal{X}}(\btheta_i,\btheta)\geq \epsilon]\leq \frac{2}{3}$. Let $\hat{\btheta}$ be generated from the following procedure: $r_i=\min\left\{r:|\mathcal{B}_{\mathcal{X}}(\btheta_i,r)\cap\mathcal{M}|\geq \frac{m}{2} \right\}$,and $\hat{\btheta}=\arg\min_{\btheta_i}r_i$. Then we have:
$$Pr[d_{\mathcal{X}}(\hat{\btheta},\btheta)\geq 3\epsilon]\leq e^{-\frac{m}{18}}$$
\end{lem}
Since the original data are $i.i.d.$ samples from underlying distribution, small group with fixed indices should also be $i.i,d.$. Therefore $\bmu_1,\bmu_2,\cdots,\bmu_k$ are $i.i.d.$. Combining Lemma~\ref{groupconstantprob} and Lemma~\ref{highmean} we get the following result:
\begin{cor}\label{highprobprojected}
The vector $\bu$ constructed in Algorithm~\ref{sparsemean} satisfies the following with probability $1-\delta$:
\begin{equation}
  \Vert \bu-G\bmu\Vert_1\leq O\left(\frac{p\log (nd/\delta)}{\epsilon}\sqrt{\frac{m}{n}}\right)
\end{equation}
\end{cor}
Then we turn to the recovery of original mean estimator. The primary tool we are using are General $M^*$ bound in~\cite{vershynin2015estimation}.
\begin{lem}[Theorem 6.2 in~\cite{vershynin2015estimation}, High Probability Version]\label{recovery}
For unknown vector $\x\in K\subseteq\mathbb{R}^d$, let $G\sim\frac{1}{\sqrt{p}}\mathcal{N}(0,1)^{p\times d}$. Noisy vector $\bnu\in \mathbb{R}^p$ with $\Vert\bnu\Vert_1\leq \sigma$. Let $\y=G\x+\bnu$. By solving the following optimization problem:
\begin{equation}
\begin{split}
\arg\min_{\x'} {\Vert \x'\Vert_{K}} \quad s.t. ~~ \Vert G\x'-\y\Vert_1\leq \sigma
\end{split}
\end{equation}
where $\norm{\cdot}_K$ denotes the Minkowski functional of $K$. Then we can get the following with probability $1-\delta$
$$\Vert \x-\x'\Vert_2\leq O\left(\frac{w(K)+\sigma+\log\frac{1}{\delta}}{\sqrt{p}}\right)$$ where $w(K)$ denotes the Gaussian width of $K$.
\end{lem}
By putting these results together we get the bound on estimation loss:
\begin{thr}
Algorithm~\ref{sparsemean} outputs $\z$ satisfying the following with probability $1-\delta$:
$$\Vert \z-\bmu\Vert_2\leq O\left(\log\frac{nd}{\delta}\sqrt{\log\frac{1}{\delta}}\left(\frac{\Lambda^2}{\epsilon^2n}\right)^{\frac{1}{4}}\right)$$
\end{thr}

\subsection{Sparse Linear Regression}

In this section, we consider empirical loss of sparse linear regression, i.e. $L(\w; D)  = \frac{1}{2n} \sum_{i=1}^n (\x_i^T\w - y_i)^2$, where $D = \{(\x_i, y_i)|i\in [n]\}, \norm{\x_i}_2 \leqslant 1, y_i \in [-1, 1].$ \footnote{Our methods suits to any radius of $\x$ and $y$.}.

Define $\w^* = \argmin_{\w \in \mc{C}} L(\w;D)$, where $\mc{C} = \{\w| \norm{\w}_1 \leqslant 1\}$. We want to obtain a vector $\w^{priv} \in \mc{C}$ within non-interactive LDP model, such that the empirical excess risk $L(\w^{priv};D) - L(\w^*; D)$ has polynomial dependences on $\log d$ and $\frac{1}{n}$. 

As in the case of high-dimensional mean estimation, directly manipulating in the original high dimensional feature space will introduce large noise, hence we use a sub-Gaussian random matrix $\Phi \in \mb{R}^{m\times d}$ to project original data (i.e. vectors in $\mb{R}^d$) into the low dimensional space (i.e. $\mb{R}^m$) first, then perturb each data in low dimensional space (i.e. Basic Private Vector mechanism given in Algorithm \ref{LDPgaussian}) which protects local privacy, and send it to the server. 

Having obtained private synopsis, the server then reconstruct an unbiased estimator for objective function according to these private synopsis. We subtract a quadratic term to ensure unbiasedness and project to PSD matrices to preserve convexity. To show good approximation guarantee, we make use of RIP bounds for random projection. As the loss function is determined by inner products between $\w$ and data, it could be uniformly preserved in projected space, which guarantees the accuracy of solution estimated with local privacy. Apparently, our methods also imply bounds with general low-complexity constraint set that preserves RIP.

Our private learning mechanism is given in Algorithm \ref{alg2} and any random projection matrix can be used here. The privacy argument directly follows from Private Vector Mechanism and composition.

\begin{algorithm}
  \caption{LDP $\ell_1$ Constrained Linear Regression}
  \label{alg2} 
  \begin{algorithmic}[1]
    \REQUIRE Personal data $(\x,y)$, parameter $\epsilon, \delta$, projection matrix $\Phi \in \mb{R}^{d\times m}$
    \ENSURE Learned classifier $\w^{priv} \in \mb{R}^d$ 
    \FOR {Each user $i=1,\dots,n$}
    \STATE $\z_i \leftarrow $ Basic Private Vector ($\Phi^T\x_i, \epsilon/2,\delta/2$)
    \STATE $v_i \leftarrow $ Basic Private Vector ($y_i, \epsilon/2, \delta/2$)
    \ENDFOR
    \STATE \parbox[t]{1.0\linewidth}{Setting $Z = [\z_1, \cdots, \z_n]^T, \sigma =  \frac{2\sqrt{2\ln(2.5/\delta)}}{\epsilon}$, \\ $ Q = \mathrm{Proj}_{\mb{S}^{+}}(Z^T Z - n\sigma^2 I_m),\bv = [v_1, \cdots, v_n]^T$}
    \STATE  \parbox[t]{1.0\linewidth}{$\w^{priv} \leftarrow \argmin_{\w \in \mc{C}} \hat{L}(\w;Z,\bv)$, where \\
    $\hat{L}(\w;Z,\bv) := \frac{1}{2n} (\Phi^T \w)^T Q (\Phi^T \w) - \frac{1}{n} \bv^T Z \Phi^T \w $} 
  \end{algorithmic}
  \end{algorithm}

  In fact, as original data is in $L_2$ ball, and random projection preserves norms with high probabilty, hence steps 2-4 in Algorithm \ref{LDPgaussian} will be executed with very low probability.

  Denote the true objective function in low dimensional space $\bar{L}(\w;\bar{X},\y) := \frac{1}{2n}\norm{\bar{X} \Phi^T \w}^2-\frac{1}{n} \y^T \bar{X} \Phi^T \w$, where $\bar{X} = [\x_1, \cdots, \x_n]^T\Phi, \w \in \mc{C}$. Let $\hat{\w}^* := \argmin_{\w \in \mc{C}} \bar{L}(\w; \bar{X},\y)$. The following lemma gives the accuracy of private solution $\w^{priv}$ when reduced into low dimensional space:
  \begin{lem}
      \label{lem4}
      Under the assumptions made in this section, given projection matrix $\Phi$, with high probability over the randomness of private mechanism, we have 
      \begin{equation}
        \bar{L}(\w^{priv};\bar{X},\y) - \bar{L}(\hat{\w}^*;\bar{X},\y) \leqslant \tilde{O}\left(\sqrt{\frac{m}{n\epsilon^2 }}\right)
      \end{equation}
  \end{lem}
  
  Now, combined with RIP bound for random projection, we can move on to prove the empirical excess risk of sparse linear regression:
  \begin{thr}
      \label{thr2} 
      Under the assumption in this section, set $m = \Theta\left(\sqrt{n\epsilon^2\log d}\right)$, then with high probability , there is 
      \begin{align*}
          L(\w^{priv}) - L(\w^*) = \tilde{O}\left(\left(\frac{\log d}{n\epsilon^2}\right)^{1/4}\right)
      \end{align*}
  \end{thr}

Note \cite{talwar2015nearly} assume data is in $L_{\infty}$ ball, while both \cite{kasiviswanathan2016efficient} and ours assume data is in $L_2$ ball. However, in LDP model, \cite{duchi2016minimax} show it was impossible to obtain polynomial dependences over $\log d$ for $\ell_0$ mean estimation problem if data is in $L_{\infty}$ ball. 

\subsection{Infinite Dimension: Kernel Ridge Regression}
Previous method mainly applies to data with finite dimensional features. However, it is common to use kernel trick in practice. This brings about new difficulties for LDP learning, as we could not add noise in the Hilbert space. In this subsection, we take kernel ridge regression as an example to show how to use Random Fourier Features (RFF) \cite{rahimi2007random} to deal with such cases caused by shift-invariant kernels (i.e. $k(\x,\y) = k(\x-\y)$). Note our technique also suits to similar problems. 

Fix a shift-invariant kernel $k(\cdot, \cdot)$, denote the Hilbert space implicitly defined as $H$, and the corresponding feature map as $\Phi: \mb{R}^d \rightarrow H$. Let the Hilbert space corresponding to the random Fourier feature map be $\hat{H} \subset \mb{R}^{d_p}$, and its feature map $\hat{\Phi}: \mb{R}^d \rightarrow \hat{H}$, where $d_p$ is the RFF projection dimension. Given a subset $\mc{X} \subset \mb{R}^d$ and data $D=\{(\x_i, y_i)|\x_i \in \mc{X}, i\in[n]\}$, for any $f \in H, g \in \hat{H}$, define loss functions in $H$ and $\hat{H}$ as follows:
\begin{align}
  L_H(f) := \frac{C}{2n}\sum_i \norm{f^T\Phi(\x_i) - y_i}_2^2 + \frac{1}{2} \norm{f}_H^2 \label{hil}\\
  L_{\hat{H}}(g) := \frac{C}{2n}\sum_i \norm{g^T\hat{\Phi}(\x_i) - y_i}_2^2 + \frac{1}{2} \norm{g}_{\hat{H}}^2  \label{proj}
\end{align} 
where $C$ is the regularization parameter. Denote $f^* = \argmin_{f \in H} L_H(f), g^* = \argmin_{g \in \hat{H}} L_{\hat{H}}(g), G$ as the Lipschitz constant of square loss, which depends on the bounded norm of features. Kernel ridge regression try to optimize formula (\ref{hil}), while after using RFF, we try to solve formula (\ref{proj}) in non-interactive LDP model, which can be easily tackled with similar mechanisms like sparse linear regression above. Borrow the key result in \cite{rubinstein2012learning} (restated in lemma \ref{uniformrff} below), which used RFF to design private mechanims for SVM in DP model, it becomes easy to prove guarantees for kernel ridge regression in our setting (see Corollary \ref{krr}). 
\begin{algorithm}
  \caption{LDP kernel mechanism}
  \label{LDPKernel}
  \begin{algorithmic}[1]
    \REQUIRE Personal data $(\x_i, y_i), i\in [n]$, random feature's dimension $d_p$, shift-invariant kernel $k(\x_1, \x_2) = k(\x_1 - \x_2)$ with Fourier transform $f(\s) = \frac{1}{2\pi}\int e^{-j\s^T\x}k(\x)d\x$, privacy parameter $\epsilon, \delta$
    \ENSURE Private output $\hat{\w}^{priv} \in \mb{R}^{d_p}$
    \STATE Draw i.i.d. samples $\s_1, \s_2, \dots, \s_{d_p} \in \mb{R}^d$ from $f(\s)$ and $b_1, b_2, \dots, b_{d_p} \in \mb{R}$ from the uniform distribution on $[0, 2\pi]$
    \FOR {$i=1,\dots,n$}
    \STATE Construct low dimensional random feature $\hat{\Phi}({\x}_i) = \sqrt{\frac{1}{d_p}}\left[\cos(\s_1^T\x_i+b_1), \dots, \cos(\s_{d_p}^T\x_i+b_{d_p})\right]' \in \hat{\mc{C}}:=\left[-\sqrt{\frac{1}{d_p}},\sqrt{\frac{1}{d_p}}\right]^{d_p} \subset \mb{R}^{d_p}$
    \STATE $\z_i \leftarrow $ Basic Private Vector ($\hat{\Phi}(\x_i), \epsilon/2,\delta/2$)
    \STATE $v_i \leftarrow $ Basic Private Vector ($y_i, \epsilon/2, \delta/2$)
    \ENDFOR
    
    \STATE \parbox[t]{1.0\linewidth}{Setting $Z = [\z_1, \cdots, \z_n]^T, \sigma =  \frac{2\sqrt{2\ln(2.5/\delta)}}{\epsilon}$, \\ $ Q = \mathrm{Proj}_{\mb{S}^{+}}(Z^T Z - n\sigma^2 I_{d_p}),\bv = [v_1, \cdots, v_n]^T$}
    \STATE  \parbox[t]{1.0\linewidth}{$\hat{\w}^{priv} \leftarrow \argmin_{\hat{\w}} \hat{L}(\hat{\w};Z,\bv)$, where \\
    $\hat{L}(\hat{\w};Z,\bv) := \frac{1}{2n}  \hat{\w}^T Q  \hat{\w} - \frac{1}{n} \bv^T Z \hat{\w} $} 
  \end{algorithmic}
  \end{algorithm}  
\begin{lem}[\cite{rubinstein2012learning}]
  \label{uniformrff}
   Suppose dual variables with respect to $f^*, g^*$ are $L_1$ norm bounded by some $r > 0$, and $\sup_{\x_1,\x_2 \in \mc{X}}|\Phi(\x_1)^T\Phi(\x_2)-(\hat{\Phi}(\x_1))^T\hat{\Phi}(\x_2)| \leqslant \gamma$, then there is $\sup_{\x \in \mc{X}}|\Phi(\x)^Tf^* - (\hat{\Phi}(\x))^Tg^*| \leqslant r\gamma + 2\sqrt{(CG+r/2)r\gamma}$.
\end{lem}
\begin{cor}
  \label{krr}
  Algorithm \ref{LDPKernel} satisfies $(\epsilon, \delta)$-LDP, and by setting $d_p = \tilde{O}\left(\sqrt{dn\epsilon^2}\right)$, with high probability, there is 
  \begin{align*}
    L_{\hat{H}}(\hat{\w}^{priv}) - & L_H(f^*) \leqslant \tilde{O}\left(\left(\frac{d}{n\epsilon^2}\right)^{1/4}\right) \\
    \sup_{\x \in \mc{X}}|\Phi(\x)^Tf^* - & (\hat{\Phi}(\x))^T\hat{\w}^{priv}| \leqslant \tilde{O}\left(\left(\frac{d}{n\epsilon^2}\right)^{1/8}\right)
  \end{align*}
\end{cor}

\section{Learning Smooth Generalized Linear Model} 

In this section, we consider learning smooth generalized linear model in non-interactive LDP setting. Non-interactive LDP learning for this problem is essentially difficult, as it is even hard to obtain an unbiased estimator of gradient. We resolve this problem using Chebyshev polynomial expansion, which requires additional smoothness assumptions. Fortunately these assumptions are naturally satisfied by a broad range of learning tasks.

We will first define the Smooth GLM loss family with appropriate assumptions. Our definition could be shown with connection to exponential family GLM, which is commonly used in machine learning. We also illustrate our algorithm and guarantees with logistic regression. 

\begin{defn}{(Absolutely Smooth Functions)}
    \label{absmooth}
  We say that an univariate function $h(x)$ is absolutely smooth, if for any $r>0$, $f(x):= h(rx)$ satisfies the following properties: there exist functions $\mu_1(k;r),\mu_2(k;r)$, which are polynomial on $k$ and $\mu_2(k;r) = O(kr)$, such that for any $k \in \mb{N}^{+}$, there is:
  \begin{itemize}
      \item[(1)] $f(x), f'(x), \dots, f^{(k-1)}(x)$ are absolutely continuous on $[-1,1]$;
      \item[(2)] $\norm{f^{(k)}(x)}_T \leqslant \mu_1(k;r) \cdot \mu_2(k;r)^k$.
    \end{itemize} 
\end{defn}

\begin{defn}{(Smooth Generalized Linear Loss, SGLL)}
  \label{SGLL}
   A loss function $\ell(\w;\x,y)$, is called smooth generalized linear loss, if for any given data $(\x,y)$, $\ell(\w;\x,y)$ is convex and $\beta$-smooth with respect to $\w$, and there exist absolutely smooth functions $h_1(x), h_2(x)$, such that $\ell(\w;\x, y) = -yh_1(\x^T\w) + h_2(\x^T\w)$.
\end{defn} 

It will be convenient to consider population risk directly. Now, we adopt standard setting of learning problems, where each data point $(\x, y)$ is drawn from some underlying unknown distribution $\mc{D}$ and $\norm{\x}_2\leqslant 1$. Given a SGLL $\ell(\w;\x, y)$, the population loss is defined as $L(\w) := \mE_{(\bm{x},y)\sim \mathcal{D}} \ell(\bm{w};\bm{x},y)$. For simplicity, instead of assuming $\w$ belongs to $B_2(r)$, we use the following equivalent notation: $\ell(\w;\x, y) = -yh_1(r\x^T\w) + h_2(r\x^T\w)$, and the constraint set for $\w$ is $\mc{C} = B_2(1)$. Denote $G(\w;\x,y) = \nabla \ell(\w;\x,y) = r m(\w;\x,y)\x$, where $m(\w;\x,y) = h_2'(r\x^T\w) - yh_1'(r\x^T\w)$. Suppose $\mE_{(\x,y)\sim \mc{D}}[\norm{G(\w;\x,y) - g(\w)}^2_2] \leqslant \sigma_0^2$, where $g(\w) = \nabla L(\w)$. This is a common assumption in stochastic optimization literature, such as \cite{bubeck2015convex}.

Given any $\alpha > 0$, we hope to design a noninteractive local DP mechanism with low sample complexity, such that the final output point $\w^{priv}$ satisfies $L(\w^{priv}) - L(\w^*) \leqslant \alpha$. 

For GLM loss functions, it is easy to see that the stochastic gradient evaluated on $w$ with data point $x_i$ is at the same direction with $x_i$. So adding isotropic noise to $x_i$ provides "unbiased" information about direction of stochastic gradient. However, the magnitude is a nonlinear function of $w^Tx_i$, making it hard for SGD even to converge to population minimizer. This is why we seek to find polynomial approximation of the magnitude of gradients.

To estimate the magnitude of gradients, we use Chebyshev polynomials to approximate nonlinear univariate function $f_{i}(x) = h'_i(rx)$, where $x \in [-1,1]$. For brevity of notations, we just use $f(x)$ to represent either $f_1(x)$ or $f_2(x)$. Denote the Chebyshev approximation with degree $p$ as $\hat{f}_{p}(x) = \frac{1}{2} + \sum_{k=1}^p a_k T_k(x)$, where $T_k(x)$ is the $k$-th Chebyshev polynomial, and $a_k = \frac{2}{\pi}\int_{-1}^1\frac{f(x)T_k(x)}{\sqrt{1-x^2}}\mr{d}x$ is the corresponding coefficient. According to existing results about Chebyshev approximations and some calculations, we have the following lemma:

\begin{lem}
  \label{chebyshev}
  Given any $\alpha > 0$, by setting $k = c\ln \frac{1}{\alpha}, p = \lceil k+e\mu_2(k;r) \rceil$, where $c$ is a constant, we have $\norm{\hat{f}_p(x) - f(x)}_{\infty} \leqslant \alpha$
\end{lem}
 The Chebyshev approximations with degree $p$ for $f_i(x)$ ($i=1,2$) are denoted as $\hat{f}_{ip}(x)=\frac{1}{2} + \sum_{k=1}^p a_{ik} T_k(x) = \sum_{k=0}^p c_{ik} x^k$, where $c_{ik}$ is the coefficient of term $x^k$. Now we approximate $m(\w;\x,y)$ and $G(\w;\x,y)$ as follows:  
\begin{align*}
    \hat{m}(\w;\x,y) := & -y\hat{f}_{1p}(r\x^T\w)+\hat{f}_{2p}(r\x^T\w) \\
    = & \sum_{k=0}^p (c_{2k} - c_{1k}y) (r\x^T\w)^k  \\
       \hat{G}(\w;\x,y) := & r\hat{m}(\w;\x,y)\x \\
\end{align*} 

With these approximations, we state our mechanism in Algorithm \ref{LDPSGLL}, where Basic Private Vector mechanism is given in Algorithm \ref{LDPgaussian}. Note an important trick in Step 6-8 of Algorithm \ref{LDPSGLL}, is that: we run basic private mechanism $p$ times, to obtain fresh private copies of the same vector $\x$, which are then used to calculate an unbiased estimation of $\hat{G}(\w;\x,y)$ with variance as low as possible (i.e. line 8 in Algorithm \ref{inexactgradient}).The LDP property of Algorithm \ref{LDPSGLL} is given as follows:
\begin{algorithm}
  \caption{LDP SGLD Mechanism - Collection}
  \label{LDPSGLL}
  \begin{algorithmic}[1]
    \REQUIRE Personal data $(\x, y)$, expansion order $p$, privacy parameter $\epsilon, \delta$
    \ENSURE Private synopsis $b = \{z_{yi}, \z_{j} |i\in \{0\}\cup[p], j\in [p(p+1)/2]\} $ sent to the server
    \STATE Setting $\epsilon_y = \frac{\epsilon}{4(p+1)},\delta_y = \frac{\delta}{4(p+1)},\epsilon_1 = \frac{\epsilon}{p(p+1)}, \delta_1 = \frac{\delta}{p(p+1)}$ 
    \STATE $\z_0 \leftarrow $ Basic Private Vector($\x, \epsilon/4, \delta/4$) 
    \FOR {$i=0,1,\dots,p $}
    \STATE $z_{yj} \leftarrow $ Basic Private Vector($y, \epsilon_y, \delta_y$)
    \ENDFOR
    \FOR {$j=1,\dots,\frac{p(p+1)}{2} $}
    \STATE $\z_{j} \leftarrow $ Basic Private Vector($\x, \epsilon_1, \delta_1$)
    \ENDFOR
  \end{algorithmic}
  \end{algorithm}
The privacy proof directly follows from Basic Vector Mechanism and Composition Theorem.
\begin{thr}
    LDP SGLD Mechanism \ref{LDPSGLL} preserves $(\epsilon, \delta)$-LDP.
\end{thr}

Having obtained the private synopsis sent by all uers, now the server can construct a stochastic inexact gradient oracle (defined in Defintion \ref{siodef}) for any point $\w\in \mc{C}$, as stated in Algorithm \ref{inexactgradient}.

\begin{defn}
    \cite{dvurechensky2016stochastic} For an objective function 
    $f(\w)$, a $(\gamma, \beta, \sigma)$ stochastic oracle returns a turple $(F_{\gamma, \beta, \sigma}(\w;\bm{\xi}), G_{\gamma, \beta, \sigma}(\w;\bm{\xi}))$, such that:
    \begin{gather*}
      \mE_{\bm{\xi}} [F_{\gamma, \beta, \sigma}(\w;\bm{\xi}) ]  =  f_{\gamma, \beta, \sigma}(\w) \\
        \mE_{\bm{\xi}}[G_{\gamma, \beta, \sigma}(\w;\bm{\xi})]   =  g_{\gamma, \beta, \sigma}(\w) \\
        \mE_{\bm{\xi}} [\norm{G_{\gamma, \beta, \sigma}(\w;\bm{\xi})  - g_{\gamma, \beta, \sigma}(\w)}^2]  \leqslant  \sigma^2  \\
        0 \leqslant h(\bv,\w) \leqslant   \frac{\beta}{2} \norm{\bv-\w}^2 + \gamma, \forall \bv,\w \in \mc{C}
    \end{gather*}
    where $h(\bv,\w) = f(\bv) - f_{\gamma, \beta, \sigma}(\w) - \innerproduct{g_{\gamma, \beta, \sigma}(\w)}{\bv-\w}$.
    \label{siodef}
 \end{defn}


\begin{algorithm}
  \caption{LDP SGLD Mechanism - Learning}
  \label{inexactgradient}
  \begin{algorithmic}[1]
    \REQUIRE Private synopsis $b= \{z_y, \z_{j} | j\in \{0\} \cup [p(p+1)/2]\}$ of each user, public coefficients $\{c_{1k}, c_{2k}| k\in \{0\} \cup [p]\}$, initial point $\w_1$ 
    \ENSURE Learned classifier $\w^{priv}$  
    \FOR{$s = 1, \dots, n$}
    \STATE $\backslash\backslash$ Construct stochastic inexact gradient 
    \STATE $\backslash\backslash$ Denote the private synopsis of user $s$ as $b$ above for abbreviation 
    \STATE Set $t_0=1$
    \FOR {$j=1,\dots,p$}
    \STATE $t_j = \prod_{i=j(j-1)/2+1}^{j(j+1)/2} (\w^T_s \z_i)$
    \ENDFOR
    \STATE $\displaystyle{\tilde{G}(\w_s;b) \leftarrow \left(\sum_{k=0}^p (c_{2k} - c_{1k}z_{yj}) t_k r^{k+1}\right)\z_0}$ 
    \STATE $\backslash\backslash$ One update via SIGM
    \STATE Run one iteration of SIGM algorithm with $\tilde{G}(\w_s, b)$ and obtain $\w_{s+1}$
    \ENDFOR
    \STATE Set $\w^{priv}:= \w_{n+1}$
    
  \end{algorithmic}
 \end{algorithm}

For any $(\x,y)$ in the domain, as loss function $\ell(\w;\x,y)$ is convex and $\beta$-smooth with respect to $\w$, we can prove the following lemma:
\begin{lem}
    \label{sovariancelem}
    For any $\gamma > 0$, setting $k = c\ln \frac{4r}{\gamma}, p = \lceil k+2\mu_2(k;r)\rceil$, then Algorithm \ref{inexactgradient} outputs a $(\gamma, \beta, \sigma)$ stochastic oracle defined in Definition \ref{siodef}, where $\sigma = \tilde{O}\left(\sigma_0+\gamma+\frac{p^{2p+1}(4r)^{p+1}}{\epsilon^{p+2}}\right)$.
\end{lem}
Based on above $(\gamma, \beta, \sigma)$ stochastic oracle, and the algorithm proposed in SIGM paper \cite{dvurechensky2016stochastic} (omitted here, due to the limitation of space), our complete learning algorithm is given in Algorithm \ref{inexactgradient}. Before proving our sample complexity, we state the basic convergence result of SIGM algorithm:
\begin{lem}[\cite{dvurechensky2016stochastic}]
    \label{sigmlem}
  Assume a function $f(\w)$ (suppose constrain set is $\mc{W}$) is endowed with a $(\gamma, \beta, \sigma)$ stochastic oracle, then the sequence $\w_k$ (corresponds to $\y_k$ in the original paper) generated by the SIGM algorithm satisfies:
  \begin{align*}
    \mE[f(\w_k)] - f(\w^*) \leqslant O\left(\frac{\sigma}{\sqrt{k}} + \gamma \right)
  \end{align*}
  where expectation is over the randomness of the stochastic oracle and $\w^* = \argmin_{\w \in \mc{W}} f(\w)$.
\end{lem}
The accuracy results directly follows from the quality of inexact stochastic gradient oracle we constructed, and the convergence result of SIGM.
\begin{thr}
    Consider smooth generalized linear loss. For any setting $\alpha > 0$, by setting $\gamma = \frac{\alpha}{2},k = c\ln \frac{4r}{\gamma}, p = \lceil k+2\mu_2(k;r)\rceil$ in Algorithm \ref{LDPSGLL}, \ref{inexactgradient}, if 
    $$n > O\left((\frac{8r}{\alpha})^{4r\ln\ln(8r/\alpha)} \left(\frac{4r}{\epsilon}\right)^{2cr\ln(8r/\alpha)+2}\left(\frac{1}{\alpha^2\epsilon^2}\right)\right),$$ we can achieve loss guarantee $L(\w^{priv}) - L(\w^*) \leqslant \alpha$
\end{thr}

As we can see, learning in non-interactive LDP model is more difficult than interactive form, especially when loss is highly nonlinear, we even can not obtain an unbiased estimation either for objective function or gradients. However, our method shows it possible to learn smooth GLM with quasi-polynomial sample complexity.

\subsection{Example: Learning Logistic Regression}
Either from the view of exponential family generalized linear model or the concrete loss function, it is not difficult to see logistic loss belongs to SGLL. For example, in logistic regression, $\ell(\w;\x, y) = \log(1+e^{-y\w^T\x}) = -\left(\frac{y}{2}\w^T\x \right) + \left(\frac{1}{2}\w^T\x + \ln(1+e^{-\w^T\x})\right)$. So we let $h_1(x) = \frac{x}{2}, h_2(x) = \frac{x}{2}+\ln(1+e^{-x})$. As we know logistic loss is convex and $\beta$-smooth for some parameter $\beta$, and the absolutely smooth property of linear function is obvious, hence once we prove $f(x) = \ln(1+e^{-x})$ is absolutely smooth, then logistic loss satisfies the definition of SGLL.

\begin{prop}
  $f(x) = \ln(1+e^{-x})$ is absolutely smooth with $\mu_1(k;r) = r\sqrt{4k\pi^3}, \mu_2(k;r) = \frac{rk}{e}$
\end{prop}

Hence, we can use private mechanisms (\ref{LDPSGLL},\ref{inexactgradient}) to learn logistic regression.  
\begin{thr}\label{logistic}
    Consider Logistic regression problem with $\ell(\w;\x,y)=\log(1+\exp(-y\w^T\x))$ For any $\alpha > 0$, by setting $\gamma = \frac{\alpha}{2},k = c\ln \frac{4r}{\gamma}, p = \lceil k+2\mu_2(k;r)\rceil$, if $n > O\left((\frac{8r}{\alpha})^{4r\ln\ln(8r/\alpha)} \left(\frac{4r}{\epsilon}\right)^{2cr\ln(8r/\alpha)+2}\left(\frac{1}{\alpha^2\epsilon^2}\right)\right)$ in Algorithm \ref{LDPSGLL}, \ref{inexactgradient}, we can achieve $L(\w^{priv}) - L(\w^*) \leqslant \alpha$. 
\end{thr}

\section{Conclusions}
In this paper, we consider how to design efficient algorithms for common learning and estimation problems under non-interactive LDP model. In particular, for sparse linear regression and mean estimation problem, we propose efficient algorithms and prove the polynomial dependence of excess risk or square error over $\log d$ and $\frac{1}{n}$, which is exactly to be expected in high dimensional case. We also extend our methods to nonparametric case and show good bounds for Kernel Ridge Regression.

For more difficult smooth generalized linear loss optimization problems, we use private Chebyshev approximations to estimate gradients of the objective loss, combined with existing inexact gradient descent methods to obtain final outputs. The sample complexity of our mechanism is quasi-polynomial with respect to $\frac{1}{\alpha}$, where $\alpha$ is the desired population excess risk.

An interesting open problem is whether our theoretical guarantees are optimal. If not, how to improve them while preserving the efficiency in non-interactive LDP model. We think these problems are critical to understand LDP in the future.

\newpage
\bibliographystyle{plain}    
\bibliography{example_paper0225}

\begin{thebibliography}{10}

\bibitem{bassily2015local}
Raef Bassily and Adam Smith.
\newblock Local, private, efficient protocols for succinct histograms.
\newblock In {\em Proceedings of the Forty-Seventh Annual ACM on Symposium on
  Theory of Computing}, pages 127--135. ACM, 2015.

\bibitem{bassily2014private}
Raef Bassily, Adam Smith, and Abhradeep Thakurta.
\newblock Private empirical risk minimization: Efficient algorithms and tight
  error bounds.
\newblock In {\em Foundations of Computer Science (FOCS), 2014 IEEE 55th Annual
  Symposium on}, pages 464--473. IEEE, 2014.

\bibitem{bubeck2015convex}
S{\'e}bastien Bubeck et~al.
\newblock Convex optimization: Algorithms and complexity.
\newblock {\em Foundations and Trends{\textregistered} in Machine Learning},
  8(3-4):231--357, 2015.

\bibitem{Chaudhuri2008Privacy}
K.~Chaudhuri and C.~Monteleoni.
\newblock Privacy-preserving logistic regression.
\newblock In {\em Conference on Neural Information Processing Systems, British
  Columbia, Canada, December}, pages 289--296, 2008.

\bibitem{chaudhuri2011differentially}
K.~Chaudhuri, C.~Monteleoni, and A.~D. Sarwate.
\newblock Differentially private empirical risk minimization.
\newblock {\em The Journal of Machine Learning Research}, 12:1069--1109, 2011.

\bibitem{dirksen2016dimensionality}
Sjoerd Dirksen.
\newblock Dimensionality reduction with subgaussian matrices: a unified theory.
\newblock {\em Foundations of Computational Mathematics}, 16(5):1367--1396,
  2016.

\bibitem{duchi2016minimax}
John Duchi, Martin Wainwright, and Michael Jordan.
\newblock Minimax optimal procedures for locally private estimation.
\newblock {\em arXiv preprint arXiv:1604.02390}, 2016.

\bibitem{duchi2013localb}
John Duchi, Martin~J Wainwright, and Michael~I Jordan.
\newblock Local privacy and minimax bounds: Sharp rates for probability
  estimation.
\newblock In {\em Advances in Neural Information Processing Systems}, pages
  1529--1537, 2013.

\bibitem{duchi2013locala}
John~C Duchi, Michael~I Jordan, and Martin~J Wainwright.
\newblock Local privacy and statistical minimax rates.
\newblock In {\em Foundations of Computer Science (FOCS), 2013 IEEE 54th Annual
  Symposium on}, pages 429--438. IEEE, 2013.

\bibitem{dvurechensky2016stochastic}
Pavel Dvurechensky and Alexander Gasnikov.
\newblock Stochastic intermediate gradient method for convex problems with
  stochastic inexact oracle.
\newblock {\em Journal of Optimization Theory and Applications},
  171(1):121--145, 2016.

\bibitem{dwork2006calibrating}
C.~Dwork, F.~McSherry, K.~Nissim, and A.~Smith.
\newblock Calibrating noise to sensitivity in private data analysis.
\newblock In {\em Theory of cryptography}, pages 265--284. Springer, New York,
  USA, 2006.

\bibitem{dwork2014algorithmic}
Cynthia Dwork and Aaron Roth.
\newblock The algorithmic foundations of differential privacy.
\newblock {\em Foundations and Trends{\textregistered} in Theoretical Computer
  Science}, 9(3--4):211--407, 2014.

\bibitem{feldman2017statistical}
Vitaly Feldman, Crist{\'o}bal Guzm{\'a}n, and Santosh Vempala.
\newblock Statistical query algorithms for mean vector estimation and
  stochastic convex optimization.
\newblock In {\em Proceedings of the Twenty-Eighth Annual ACM-SIAM Symposium on
  Discrete Algorithms}, pages 1265--1277. Society for Industrial and Applied
  Mathematics, 2017.

\bibitem{geng2014optimal}
Quan Geng and Pramod Viswanath.
\newblock The optimal mechanism in differential privacy.
\newblock In {\em Information Theory (ISIT), 2014 IEEE International Symposium
  on}, pages 2371--2375. IEEE, 2014.

\bibitem{Hardt2012A}
M.~Hardt, K.~Ligett, and F.~Mcsherry.
\newblock A simple and practical algorithm for differentially private data
  release.
\newblock In {\em Advances in Neural Information Processing Systems}, pages
  2339--2347, 2012.

\bibitem{Hardt2010A}
M.~Hardt and G.~N. Rothblum.
\newblock A multiplicative weights mechanism for privacy-preserving data
  analysis.
\newblock In {\em IEEE Symposium on Foundations of Computer Science}, pages
  61--70, 2010.

\bibitem{hsu2016loss}
Daniel Hsu and Sivan Sabato.
\newblock Loss minimization and parameter estimation with heavy tails.
\newblock {\em Journal of Machine Learning Research}, 17(18):1--40, 2016.

\bibitem{kairouz2016discrete}
Peter Kairouz, Keith Bonawitz, and Daniel Ramage.
\newblock Discrete distribution estimation under local privacy.
\newblock In {\em Proceedings of The 33rd International Conference on Machine
  Learning}, pages 2436--2444, 2016.

\bibitem{kairouz2014extremal}
Peter Kairouz, Sewoong Oh, and Pramod Viswanath.
\newblock Extremal mechanisms for local differential privacy.
\newblock In {\em Advances in neural information processing systems}, pages
  2879--2887, 2014.

\bibitem{kairouz2015composition}
Peter Kairouz, Sewoong Oh, and Pramod Viswanath.
\newblock The composition theorem for differential privacy.
\newblock In {\em Proceedings of The 32nd International Conference on Machine
  Learning}, pages 1376--1385, 2015.

\bibitem{kairouz2015secure}
Peter Kairouz, Sewoong Oh, and Pramod Viswanath.
\newblock Secure multi-party differential privacy.
\newblock In {\em Advances in Neural Information Processing Systems}, pages
  2008--2016, 2015.

\bibitem{Kasiviswanathan2010What}
S.~P. Kasiviswanathan, H.~K. Lee, K.~Nissim, S.~Raskhodnikova, and A.~Smith.
\newblock What can we learn privately?
\newblock In {\em IEEE Symposium on Foundations of Computer Science}, pages
  531--540, 2008.

\bibitem{kasiviswanathan2016efficient}
Shiva~Prasad Kasiviswanathan and Hongxia Jin.
\newblock Efficient private empirical risk minimization for high-dimensional
  learning.
\newblock In {\em Proceedings of The 33rd International Conference on Machine
  Learning}, pages 488--497, 2016.

\bibitem{kasiviswanathan2011can}
Shiva~Prasad Kasiviswanathan, Homin~K Lee, Kobbi Nissim, Sofya Raskhodnikova,
  and Adam Smith.
\newblock What can we learn privately?
\newblock {\em SIAM Journal on Computing}, 40(3):793--826, 2011.

\bibitem{kearns1998efficient}
Michael Kearns.
\newblock Efficient noise-tolerant learning from statistical queries.
\newblock {\em Journal of the ACM (JACM)}, 45(6):983--1006, 1998.

\bibitem{kifer2012private}
Daniel Kifer, Adam Smith, and Abhradeep Thakurta.
\newblock Private convex empirical risk minimization and high-dimensional
  regression.
\newblock {\em Journal of Machine Learning Research}, 1(41):3--1, 2012.

\bibitem{lei2011differentially}
J.~Lei.
\newblock Differentially private m-estimators.
\newblock In {\em Advances in Neural Information Processing Systems}, pages
  361--369, 2011.

\bibitem{qazi2007some}
MA~Qazi and QI~Rahman.
\newblock Some coefficient estimates for polynomials on the unit interval.
\newblock {\em Serdica Mathematical Journal}, 33(4):449p--474p, 2007.

\bibitem{rahimi2007random}
Ali Rahimi, Benjamin Recht, et~al.
\newblock Random features for large-scale kernel machines.
\newblock In {\em NIPS}, volume~3, page~5, 2007.

\bibitem{rubinstein2012learning}
B.~Rubinstein, P.~L. Bartlett, L.~Huang, and N.~Taft.
\newblock Learning in a large function space: Privacy-preserving mechanisms for
  svm learning.
\newblock {\em Journal of Privacy and Confidentiality}, 4(1):4, 2012.

\bibitem{Smith2011Privacy}
A.~Smith.
\newblock Privacy-preserving statistical estimation with optimal convergence
  rates.
\newblock In {\em ACM Symposium on Theory of Computing, STOC}, pages 813--822,
  2011.

\bibitem{smith2013differentially}
Adam Smith and Abhradeep Thakurta.
\newblock Differentially private model selection via stability arguments and
  the robustness of the lasso.
\newblock {\em J Mach Learn Res Proc Track}, 30:819--850, 2013.

\bibitem{talwar2015nearly}
Kunal Talwar, Abhradeep Thakurta, and Li~Zhang.
\newblock Nearly optimal private lasso.
\newblock In {\em Advances in Neural Information Processing Systems}, pages
  3025--3033, 2015.

\bibitem{Thaler2012Faster}
J.~Thaler, J.~Ullman, and S.~Vadhan.
\newblock Faster algorithms for privately releasing marginals.
\newblock In {\em International Colloquium on Automata, Languages, and
  Programming}, volume 7391, pages 810--821. 2012.

\bibitem{trefethen2008gauss}
Lloyd~N Trefethen.
\newblock Is gauss quadrature better than clenshaw--curtis?
\newblock {\em SIAM review}, 50(1):67--87, 2008.

\bibitem{tropp2015introduction}
Joel~A Tropp et~al.
\newblock An introduction to matrix concentration inequalities.
\newblock {\em Foundations and Trends{\textregistered} in Machine Learning},
  8(1-2):1--230, 2015.

\bibitem{vershynin2009note}
Roman Vershynin.
\newblock A note on sums of independent random matrices after ahlswede-winter.
\newblock {\em Lecture notes}, 2009.

\bibitem{vershynin2015estimation}
Roman Vershynin.
\newblock Estimation in high dimensions: a geometric perspective.
\newblock In {\em Sampling theory, a renaissance}, pages 3--66. Springer, 2015.

\bibitem{wangyuxiang2015}
Y.~Wang, S.~E. Fienberg, and A.~J. Smola.
\newblock Privacy for free: Posterior sampling and stochastic gradient monte
  carlo.
\newblock In {\em International Conference on Machine Learning}, pages
  2493--2502, 2015.

\bibitem{Wang2016Differentially}
Z.~Wang, C.~Jin, K.~Fan, J.~Zhang, J.~Huang, Y.~Zhong, and L.~Wang.
\newblock Differentially private data releasing for smooth queries.
\newblock {\em Journal of Machine Learning Research}, 17(51):1--42, 2016.

\bibitem{warner1965randomized}
Stanley~L Warner.
\newblock Randomized response: A survey technique for eliminating evasive
  answer bias.
\newblock {\em Journal of the American Statistical Association},
  60(309):63--69, 1965.

\bibitem{zhang2017efficient}
Jiaqi Zhang, Kai Zheng, Wenlong Mou, and Liwei Wang.
\newblock Efficient private erm for smooth objectives.
\newblock {\em arXiv preprint arXiv:1703.09947}, 2017.

\bibitem{zhang2012functional}
Jun Zhang, Zhenjie Zhang, Xiaokui Xiao, Yin Yang, and Marianne Winslett.
\newblock Functional mechanism: regression analysis under differential privacy.
\newblock {\em Proceedings of the VLDB Endowment}, 5(11):1364--1375, 2012.

\end{thebibliography}

\setcounter{section}{0}
\appendix
\renewcommand{\appendixname}{Appendix~\Alph{section}}
\section{Appendix} 
\label{sec:appendix}

\subsection{Omitted Proofs in Section 3}
{
\renewcommand{\thelem}{\ref{groupconstantprob}}
\begin{lem}
  Let $\x_1,\x_2,\cdots,\x_n\sim i.i.d.\mathcal{D}$ with $\bmu=\mathbb{E}_{\mathcal{D}}[\x]$ and $supp(\mathcal{D})\subseteq \mathcal{B}(0,1)$. Let $G$ and $\{\y_i\}_{i=1}^n$ defined in the above procedure. For each of group $S_j$ fixed, we have the following with probability $2/3$:
  \begin{equation}
  \Big\Vert \frac{1}{|S_j|}\sum_{\y_i\in S_j}\y_i-G\bmu\Big\Vert_1\leq O\left(\frac{p\log (nd)}{\epsilon\sqrt{|S_j|}}\right)
  \end{equation}
\end{lem}
\addtocounter{lem}{-1}
} 
\begin{proof}
Apparently $\frac{1}{|S_j|}\sum_{i\in S_j}\br_i\sim \mathcal{N}(0,\frac{2\log (1.25/\delta)}{|S_j|\epsilon^2}I_d)$. So we have $\Vert\frac{1}{|S_j|}\sum_{i\in S_j}\br_i\Vert_1\leq O\left(\frac{p\log n}{\epsilon\sqrt{|S_j|}}\right)$ with probability $\frac{1}{9}$. We then turn to bound the loss incurred by random sample of data.
\begin{equation}
\begin{split}
&\mathbb{E}\Big\Vert \bmu-\frac{1}{|S_j|}\sum_{i\in S_j}\x_i\Big\Vert^2=\frac{1}{|S_j|}\sum_{l=1}^d\mathrm{var}(x_{1l})\\
\leq &\frac{1}{|S_j|}\sum_{l=1}^d\mathbb{E}[x_{1l}^2]\leq \frac{1}{|S_j|}.
\end{split}
\end{equation}
According to Markov Inequality, we have
$$\mathcal{P}\left\{\Big\Vert \bmu-\frac{1}{|S_j|}\sum_{i\in S_j}\x_i\Big\Vert^2\geq \frac{9}{|S_j|}\right\}\leq \frac{1}{9}$$
Given $\x_1,\x_2,\cdots,\x_n$ fixed under this event, we can easily derive upper bounds on entries of $ G(\bmu-\frac{1}{|S_j|}\sum_{i\in S_j}\x_i)$: for $\g\sim\mathcal{N}(0,I_d)$ and $\q=\bmu-\frac{1}{|S_j|}\sum_{i\in S_j}\x_i$, we have $|\g^T\q|\leq 12\sqrt{\frac{\log d}{|S_j|}}$ with probability $1-\frac{1}{9d}$. By union bound we have the following with probability $\frac{2}{9}$:
$$\Big\Vert G(\bmu-\frac{1}{|S_j|}\sum_{i\in S_j}\x_i)\Big\Vert_1\leq O\left(\sqrt{\frac{p\log d}{|S_j|}}\right).$$
Putting the two inequalities together using union bound, we get the result.
\end{proof}
{
\renewcommand{\thelem}{\ref{lem4}}
 \begin{lem}
      Under the assumptions made in Section 3.2, given projection matrix $\Phi$, with high probability over the randomness of private mechanism, we have 
      \begin{equation}
        \bar{L}(\w^{priv};\bar{X},\y) - \bar{L}(\hat{\w}^*;\bar{X},\y) \leqslant \tilde{O}\left(\sqrt{\frac{m}{n\epsilon^2 }}\right)
      \end{equation}
  \end{lem}
  \addtocounter{lem}{-1}
}
  
\begin{proof}
      Note, once we prove the uniform convergence of $|\hat{L}(\w;Z,\bv)-\bar{L}(\w;\bar{X}, \y)| \leqslant O\left(\sqrt{\frac{m}{n\epsilon^2}}\right) $ for any $\w \in \mc{C}$, then the conclusion holds directly. Now, we will prove the uniform convergence. Note $Z = \bar{X} + E$, where $E \in \mb{R}^{n\times m}$, and each entry $e_{ij} \sim \mc{N}(0, \sigma^2)$, $\bv = \y + \br$, where $\br \sim \mc{N}(0, \sigma^2 I_n)$. Denote $\bar{\w} = \Phi^T \w$.
      \begin{align*}
        & \left|\hat{L}(\w;Z,\bv) - \bar{L}(\w;\bar{X}, \y)\right| \\
        = & \left|\frac{1}{2n} \bar{\w}^T (Q - \bar{X}^T \bar{X})\bar{\w}  - \frac{1}{n} \left(\bv^T Z \bar{\w} - \y^T \bar{X} \bar{w}\right)\right| \\
        \leqslant & \frac{1}{2n} \norm{Q - \bar{X}^T \bar{X}}_2 \norm{\bar{\w}}_2^2  +  \frac{1}{n} \left|\bv^T Z \bar{\w} -  \y^T \bar{X} \bar{\w}\right| \\
        \leqslant & \frac{1}{2n} \norm{Q - \bar{X}^T \bar{X}}_F \norm{\bar{\w}}_2^2  +  \frac{1}{n} |\bv^T Z \bar{\w} -  \y^T \bar{X} \bar{\w}| \\
        \leqslant & \Scale[0.95]{\frac{1}{2n} \norm{Z^T Z - n\sigma^2 I_m - \bar{X}^T \bar{X}}_F \norm{\bar{\w}}_2^2  +  \frac{1}{n} |\bv^T Z \bar{\w} -  \y^T \bar{X} \bar{\w}|} \\
        \leqslant & \frac{1}{2n} \norm{E^T E - n\sigma^2 I_m}_F\norm{\bar{\w}}^2_2 + \frac{1}{n} \norm{\bar{X}^T E}_F\norm{\bar{\w}}^2_2 +\\
        & \frac{1}{n} \left(\norm{E^T \y}_2 +  \norm{\bar{X}^T \br}_2 +\norm{E^T \br}_2\right)\norm{\bar{\w}}_2 
      \end{align*}
      From the property of random projection, we know $\norm{\bar{\w}}_2 \leqslant 1$ with high probability. Besides, as each entry in $E$ is i.i.d. Gaussian, and $\mb{E}[E^TE] = n\sigma^2 I_m$, thus we have  $\frac{1}{2n} \norm{E^T E - n\sigma^2 I_m}_2 \leqslant O\left(\sigma \sqrt{\frac{\log m}{n}}\right)$ with high probability according to lemma \ref{matrixcon}, hence $\frac{1}{2n} \norm{E^T E - n\sigma^2 I_m}_F \leqslant O(\sigma \sqrt{\frac{m\log m }{n}})$ with high probability.

      As $\frac{1}{n^2}\norm{\bar{X}^T E}_F^2 = \frac{1}{n^2}\sum_{j=1}^m(\sum_{i=1}^m (\q_j^T\bm{\e}_i)^2)$, where $\q_j, \bm{e}_i$ are the $j$-th and $i$-th column of $\bar{X}$ and $E$ respectively. For each $j\in[m]$, $\frac{1}{n^2} \sum_{i=1}^m (\q_j^T\bm{e}_i)^2$ obeys Chi-square distribution (with some scaling), thus with high probability, $\frac{1}{n^2} \sum_{i=1}^m (\q_j^T\bm{\e}_i)^2 \leqslant O\left(\frac{m\norm{\q_j}^2 \sigma^2}{n^2}\right)$. Therefore, by union bound, we have $\frac{1}{n^2}\sum_{j=1}^m(\sum_{i=1}^m (\q_j^T\bm{\e}_i)^2) \leqslant O\left(\frac{m\sum_j\norm{\q_j}^2\sigma^2}{n^2}\right) = O\left(\frac{m\sigma^2}{n}\right)$, as $\sum_j\norm{\q_j}^2 = \norm{\bar{X}}_F^2 \leqslant n$. Hence, there is $\frac{1}{n} \norm{\bar{X}^T E}_F \leqslant O\left(\sqrt{\frac{m\sigma^2}{n}}\right)$ with high probability. Using similar augument, we have $\frac{1}{n} \norm{\bar{E}^T \y}_2 \leqslant O\left(\sqrt{\frac{m\sigma^2}{n}}\right), \frac{1}{n} \norm{\bar{E}^T \br}_2 \leqslant O\left(\sqrt{\frac{m\sigma^2}{n}}\right)$ with high probability. For $\frac{1}{n}\norm{\bar{X}^T r}$, according to matrix concentration inequality (Theorem 4.1.1 in \cite{tropp2015introduction}), we have $\frac{1}{n} \norm{\bar{X}^T \br}_2 \leqslant O\left(\frac{1}{\sqrt{n}}\right)$. 

      Combine all these results together, we obtain the desired conclusion.
  \end{proof}
  \begin{lem}[\cite{vershynin2009note}]
  \label{matrixcon}
  Suppose $\x \in \mb{R}^d$ be a random vector satisfies $\mb{E}[\x\x^T] = I_d$. Denote $\norm{\x}_{\phi_1} = M$, where $\norm{\cdot}_{\psi_1}$ represents Orlicz $\psi_1$-norm. Let $\x_1, \dots, \x_n$ be independent copies of $\x$, then for every $\epsilon \in (0,1)$, we have 
  \begin{align*}
    \pr\left(\norm{\frac{1}{n}\sum_{i=1}^n\x_i\x_i^T - I_d}_2 > \epsilon\right) \leqslant de^{-n\epsilon^2/4M^2}
  \end{align*}
\end{lem}

{
\renewcommand{\thethr}{\ref{thr2}}
 \begin{thr}
      Under the assumption in this section, set $m = \Theta\left(\sqrt{n\epsilon^2\log d}\right)$ for $\beta > 0$, then with high probability , there is 
      \begin{align*}
          L(\w^{priv}) - L(\w^*) = \tilde{O}\left(\left(\frac{\log d}{n\epsilon^2}\right)^{1/4}\right)
      \end{align*}
  \end{thr}
  \addtocounter{thr}{-1}
}
  \begin{proof}
      On one hand,
      \begin{align*}
          & L(\w^{priv}) - L(\w^*) \\
          = & L(\w^{priv}) - \bar{L}(\w^{priv}) +  \bar{L}(\w^{priv}) - \bar{L}(\hat{\w}^*)  \\
          & + \bar{L}(\hat{\w}^*) - \bar{L}(\w^*) + \bar{L}(\w^*) - L(\w^*) \\
          \leqslant & \left[L(\w^{priv}) - \bar{L}(\w^{priv}) + \bar{L}(\w^*) - L(\w^*)\right] \\
          & + \bar{L}(\w^{priv}) - \bar{L}(\hat{\w}^*) \\
          \leqslant & G[\max_i\{|\innerproduct{\w^{priv}}{\x_i} - \innerproduct{\Phi^T\w^{priv}}{\Phi^T\x_i}|\} \\
          & +\max_i\{|\innerproduct{\w^*}{\x_i} - \innerproduct{\Phi^T\w^*}{\Phi^T\x_i}|\}] \\
          & + [\bar{L}(\w^{priv}) - \bar{L}(\hat{\w}^*)]  \numberthis \label{eq1}\\
          & \quad \quad (\text{where $G$ is the Lipschitz constant}) 
      \end{align*}

      On the other hand, for $\forall \w \in \mc{C}, \forall \x \in D$, there is 
      \begin{align*}
        & |\innerproduct{\w}{\x} - \innerproduct{\Phi^T\w}{\Phi^T\x}| \\
        = & \Scale[1.2]{\left|\frac{\norm{\Phi^T(\w+\x)}_2^2 - \norm{\Phi^T(\w-\x)}_2^2}{4} - \frac{\norm{\w+\x}_2^2 - \norm{\w-\x}_2^2}{4} \right|}    \\
        \leqslant & \Scale[1.2]{\left|\frac{\norm{\Phi^T(\w+\x)}_2^2 - \norm{\w+\x}_2^2}{4} \right| +\left|\frac{\norm{\Phi^T(\w-\x)}_2^2 - \norm{\w-\x}_2^2}{4} \right|}
      \end{align*}

      According to the results of random projection w.r.t. additive error \cite{dirksen2016dimensionality}, we know with high probability, there is $|\innerproduct{\w}{\x} - \innerproduct{\Phi^T\w}{\Phi^T\x}| \leqslant O\left(\sqrt{\frac{\log d}{m}}\right)$, for $\forall \w \in \mc{C}, \forall \x \in D$. Therefore, the first term in equation (\ref{eq1}) is less than $O\left(\sqrt{\frac{\log d}{m}}\right)$.

      From lemma \ref{lem4}, we know $\bar{L}(\bar{\w}^{priv}) - \bar{L}(\bar{\w}^*) \leqslant \tilde{O}\left(\sqrt{\frac{m}{n\epsilon^2 }}\right)$ holds with high probability. Combine these two inequalities, it is easy to determine the optimal $m$, then obtain the conclusion.
  \end{proof}

{
\renewcommand{\thecor}{\ref{krr}}
\begin{cor}
  Algorithm LDP kernel mechanism satisfies $(\epsilon, \delta)$-LDP, and with high probability, there is 
  \begin{align*}
    L_{\hat{H}}(\hat{\w}^{priv}) - L_H(f^*) \leqslant \tilde{O}\left(\left(\frac{d}{n\epsilon^2}\right)^{1/4}\right) \\
    \sup_{\x \in \mc{X}}|\Phi(\x)^Tf^* - (\hat{\Phi}(\x))^T\hat{\w}^{priv}| \leqslant \tilde{O}\left(\left(\frac{d}{n\epsilon^2}\right)^{1/8}\right)
  \end{align*}
\end{cor}
\addtocounter{cor}{-1}
}
\begin{proof}
    Algorithm satisfies local privacy is obvious. For excess risk, as $L_{\hat{H}}(\hat{\w}^{priv}) - L_H(f^*) = L_{\hat{H}}(\hat{\w}^{priv}) - L_{\hat{H}}(g^*)+L_{\hat{H}}(g^*) - L_H(f^*)$, follow nearly the same proof of lemma 5 of sparse linear regression, we have $L_{\hat{H}}(\hat{\w}^{priv}) - L_{\hat{H}}(g^*) \leqslant \tilde{O}\left(\sqrt{\frac{d_p}{n\epsilon^2 }}\right)$. On the other hand, nearly borrow the proof of Lemma 17 in \cite{rubinstein2012learning} and property of RRF , we have 
    \begin{align*}{}
      L_{\hat{H}}(g^*) - L_H(f^*) \leqslant \tilde{O}\left( \sqrt{\frac{d}{d_p}}\right)
    \end{align*}
  Combine above two inequalities, and choose optimal $d_p$ as $\tilde{O}\left(\sqrt{dn\epsilon^2}\right)$, we obtain the first inequality of the conclusion. Then combine lemma 7 in this paper, it is easy to obtaint the second inequality.
\end{proof}

\subsection{Omitted contents and proofs in Section 4}
\subsubsection{Relations between smooth generalized linear losses (SGLL) and generalized linear models (GLM)}
  Note that a model is called GLM, if for $\bm{x}, \w^* \in \mb{R}^d$, label $y$ with respect to $\bm{x}$ is given by a distribution which belongs to the exponential family:
  \begin{align}
    p(y|\bm{x}, \bm{w}^*) = \exp\left(\frac{y\theta - b(\theta)}{\Phi} + c(y, \Phi)\right)
  \end{align}
  where $\theta, \Phi$ are parameters, and $b(\theta),c(y, \Phi)$ are known functions. Besides, there is an one-to-one continuous differentiable transformation $g(\cdot)$ such that $g(b'(\theta)) = \bm{x}^T \bm{w}^*$. 

  According to the key equality $g(b'(\theta)) = \bm{x}^T \bm{w}^*$, usually we can obtain smooth function $\theta = h_1(\bm{x}^T \bm{w}^*), b(\theta) = h_2(\bm{x}^T \bm{w}^*)$, and what's more, univariate function $h_i(x) (i=1,2)$ satisfies the absolutely smooth property. 

  For such GLM, if we consider optimizing the expected negative logarithmic probability $-\mE_{(\bm{x},y)\sim \mathcal{D}} \log p(\x, y;\w)$, once discarding unrelated terms to $\w$, we obtain the new population loss, $L(\w) := \mE_{(\bm{x},y)\sim \mathcal{D}} \ell(\bm{w};\bm{x},y)$, where $\ell(\bm{w};\bm{x}, y) = -yh_1(\x^T\w) + h_2(\x^T\w)$, exactly the form of smooth generalized linear loss defined in section 4. Hence our SGLL is a natural loss defined by GLM with additional smoothness assumptions.

\subsubsection{Omitted proofs}
{
\renewcommand{\thelem}{\ref{chebyshev}}
\begin{lem}
  Given any $\alpha > 0$, by setting $k = c\ln \frac{1}{\alpha}, p = \lceil k+e\mu_2(k;r) \rceil$, where $c$ is a constant, we have $\norm{\hat{f}_p(x) - f(x)}_{\infty} \leqslant \alpha$.
\end{lem}
\addtocounter{lem}{-1}
}
\begin{proof}
  As $f, f', \cdots, f^{(k-1)}$ are absolutely continuous over $[-1,1]$, and $\norm{f^{(k)}}_T \leqslant \mu_1(k;r)\mu_2(k;r)^k$, according to the results in \cite{trefethen2008gauss}, we have 
  \begin{align*}
    \norm{\hat{f}_p(x) - f(x)}_{\infty} \leqslant & \frac{2\norm{f^{(k)}}_T}{\pi k(p-k)^k} \\
    \leqslant & \frac{2\mu_1(k;r)}{\pi ke^{k}} \numberthis \label{eq22}
  \end{align*}
 It is easy to see there exists $c>0$, such that the term (\ref{eq22}) is less than $\alpha$ with chosen $k$, hence the conclusion holds.
\end{proof}

{
\renewcommand{\thelem}{\ref{sovariancelem}}
\begin{lem}
    For any $\gamma > 0$, setting $k = c\ln \frac{4r}{\gamma}, p = \lceil k+2\mu_2(k;r)\rceil$, then algorithm 7 outputs a $(\gamma, \beta, \sigma)$ stochastic oracle, where $\sigma = \tilde{O}\left(\sigma_0+\gamma+\frac{p^{2p+1}(4r)^{p+1}}{\epsilon^{p+2}}\right)$.
\end{lem}
\addtocounter{lem}{-1}
}

\begin{proof}
    According to lemma \ref{chebyshev},  we know the approximation error, $|\hat{m}(\w;\x,y)-m(\w;\x,y)| \leqslant \frac{\gamma}{2r}$. 
    For any fixed $(\x,y)$, from the construction of stochastic inexact gradient oracle, there is $\mb{E}[\tilde{G}(\w;b)|\x,y] = \hat{G}(\w;\x,y)$. Denote $\hat{g}(\w) = \mE_{(\x,y)\sim \mc{D}}[\hat{G}(\w;\x,y)]$, thus we have 
    \begin{align*}
        \mE\left[\norm{\tilde{G}(\w;b) - \hat{g}(\w)}^2\right] = & \mE\left[\norm{\tilde{G}(\w;b) - \hat{G}(\w;\x,y)}^2\right] \\
        & + \mE\left[\norm{\hat{G}(\w;\x,y) - \hat{g}(\w)}^2\right]
    \end{align*}

    For above two terms, combined with results given in lemma \ref{sigvar}, we we obtain $$\mE\left[\norm{\tilde{G}(\w;b) - g(\w)}^2\right] \leqslant \tilde{O}\left(\left(\frac{r(2rp)^{p+1}}{\epsilon^{p+2}} + \gamma + \sigma_0 \right)^2 \right)$$.
   
    As $L(\bv) - L(\w) - \hat{g}(\w)^T(\bv-\w) = L(\bv) - L(\w) - g(\w)^T(\bv - \w) + (g(\w) - \hat{g}(\w))^T(\bv-\w)$, and from the approximation error, we know $|(g(\w) - \hat{g}(\w))^T(\bv-\w)| \leqslant  \frac{\gamma}{2}$. What's more, as $L(\w)$ is convex and $\beta$-smooth, that is $0 \leqslant L(\bv) - L(\w) - g(\w)^T(\bv-\w) \leqslant \frac{\beta}{2}\norm{\bv-\w}^2$. Combined these inequalities, we obtain 
    \begin{align*}
        & \Scale[0.9]{-\frac{\gamma}{2} \leqslant L(\bv) - L(\w) - \hat{g}(\w)^T(\bv-\w) \leqslant \frac{\beta}{2}\norm{\bv-\w}^2 + \frac{\gamma}{2}} \\
        \Longleftrightarrow & \Scale[0.9]{0 \leqslant L(\bv) - (L(\w)-\frac{\gamma}{2}) - \hat{g}(\w)^T(\bv-\w) \leqslant \frac{\beta}{2}\norm{\bv-\w}^2 + \gamma} \\
    \end{align*}
    Note the function value oracles in the stochastic oracle definition (either $F_{\gamma, \beta, \sigma}(\cdot)$ or $f_{\gamma, \beta, \sigma}(\cdot)$) do not play any role in the optimization algorithm, hence we can set it as $L(\w)-\frac{\gamma}{2}$, though we do not know how to calculate.
\end{proof}

\begin{lem}
\label{sigvar}
  Based on above statements, we have 
  \begin{align*}
    \mE\left[\norm{\tilde{G}(\w;b) - \hat{G}(\w;\x,y)}^2\right] & \leqslant \tilde{O}\left(\frac{p^{4p+2}(4r)^{2p+2}}{\epsilon^{2p+4}}\right) \\
    \mE\left[\norm{\hat{G}(\w;\x,y) - \hat{g}(\w)}^2\right] & \leqslant (\gamma + \sigma_0)^2 
  \end{align*}
\end{lem}
\begin{proof}
  First, we calculate the variance of each $t_k$, $\mathrm{var}(t_j) \leqslant \prod_{i=j(j-1)/2+1}^{j(j+1)/2} (\mr{var}(\w^T\z_i) + (\mE[\w^T\z_i])^2) \leqslant \tilde{O}\left((\frac{p(p+1)}{\epsilon})^{2j}\right)$.

  Next, we upper bound the coefficient $c_k$ (as it is the same for $c_{1k}$ and $c_{2k}$, hence we use $c_k$ for short). Note $c_k = \sum_{m=k}^p a_mb_{mk}$, where $a_m$ is the coefficient of original function represented by Chebyshev basis, $b_{mk}$ is the coefficient of order $k$ monomial in Chebyshev basis $T_m(x)$, where $0 \leqslant k \leqslant m$. According to the formula of $T_m(x)$ given in \cite{qazi2007some} and well-known Stirling's approximation, after some translation, we have 
  \begin{align*}
    |b_{mk}| \leqslant & \max_{\theta \in (0, \frac{1}{2})} O\left(\sqrt{m} \cdot \left[\frac{(1-\theta)^{1-\theta}}{\theta^{\theta}(1-2\theta)^{1-2\theta}}\right]^m\right) \\
    \leqslant & O\left(\sqrt{m} 2^m\right)
  \end{align*}
  Besides, from the absolutely smooth property of $h'_i(x) (i \in \{1,2\})$ and the convergence results in \cite{trefethen2008gauss}, we have $a_m \leqslant O\left(\frac{1}{m^2}\right)$, thus $c_k = \sum_{m=k}^p a_mb_{mk} \leqslant O\left(2^p\right)$. Hence, there is 
  \begin{align*}
    \mr{var}\left[(c_{2k} - c_{1k}z_y) t_k r^{k+1}\right] \leqslant & r^{2k+2}\mE \left[\left((c_{2k} - c_{1k}z_y) t_k\right)^2\right] \\
    \leqslant & O\left(\frac{p^{4k+2}(4r)^{2p+2}}{\epsilon^{2k+2}}\right)
  \end{align*}
  As each $(c_{2k} - c_{1k}z_y) t_k r^{k+1}$ is independent with each other (for different $k$), which leads to 
  \begin{align*}
    \mr{var}\left[\sum_{k=0}^p(c_{2k} - c_{1k}z_y) t_k r^{k+1}\right] \leqslant O\left(\frac{p^{4p+2}(4r)^{2p+2}}{\epsilon^{2p+2}}\right)
  \end{align*}
  Moreover, $\mr{var}(\z_0) \leqslant O\left(\frac{1}{\epsilon^2}\right)$. Therefore, 
  \begin{align*}
    \mE\left[\norm{\tilde{G}(\w;b) - \hat{G}(\w;\x,y)}^2\right] \leqslant \tilde{O}\left(\frac{p^{4p+2}(4r)^{2p+2}}{\epsilon^{2p+4}}\right)
  \end{align*}
  For second inequality in the conclusion, there is 
    \begin{align*}
        & \mE\left[\norm{\hat{G}(\w;\x,y) - \hat{g}(\w)}^2\right] \\
        \leqslant & \Scale[0.8]{\mE\left[\norm{\hat{G}(\w;\x,y) - G(\w;\x,y) + G(\w;\x,y) - g(\w) + g(\w) - \hat{g}(\w)}^2\right]} \\
        \leqslant & \gamma^2 + \sigma_0^2 + 2\sigma_0\gamma = (\gamma + \sigma_0)^2 
    \end{align*}
\end{proof}

\begin{prop}
  $f(x) = \ln(1+e^{-x})$ is absolutely smooth with $\mu_1(k;r) = r\sqrt{4k\pi^3}, \mu_2(k;r) = \frac{rk}{e}$
\end{prop}
\begin{proof}
  For any $r,k>0$, the absolutely continuous of $f^{(k)}(rx)$ is obvious, now consider $\norm{f^{(k+1)}(rx)}_T$:
  \begin{align*}
      \norm{f^{(k+1)}}_T = & \int_{-1}^1\frac{|f^{(k+2)}(rx)|}{\sqrt{1-x^2}}\mr{d}x \\
      \leqslant & \pi \norm{f^{(k+2)}(rx)}_{\infty} \\
      \leqslant & \Scale[0.9]{\pi r^{k+2} \norm{\sum_{j=1}^{k+1}(-1)^{k+j}A_{k+1,j-1} f^j(1-f)^{k+2-j}}_{\infty}} \\
      \leqslant & \pi r^{k+2} \sum_{j=1}^{k+1} A_{k+1, j-1} \\
      \leqslant & \pi (k+1)! r^{k+2} \\
      \leqslant & \sqrt{4\pi^3} r^{k+2} (k+1)^{k+3/2} e^{-k-1} \\
      = & r\sqrt{4\pi^3(k+1)} \left(\frac{r(k+1)}{e}\right)^{k+1}
    \end{align*}
\end{proof}

{
\renewcommand{\thethr}{\ref{logistic}}
\begin{thr}
    For any $\alpha > 0$, set $\gamma = \frac{\alpha}{2},k = c\ln \frac{4r}{\gamma}, p = \lceil k+2\mu_2(k;r)\rceil$, if \\$n > O\left((\frac{8r}{\alpha})^{4r\ln\ln(8r/\alpha)} \left(\frac{4r}{\epsilon}\right)^{2cr\ln(8r/\alpha)+2}\left(\frac{1}{\alpha^2\epsilon^2}\right)\right)$, using Algorithms \ref{LDPSGLL} and \ref{inexactgradient}, then we have $L(\w^{priv}) - L(\w^*) \leqslant \alpha$. 
\end{thr}
\addtocounter{thr}{-1}
}
\begin{proof}
    According to lemma \ref{sigmlem} , with a $(\gamma, \beta, \sigma)$ stochastic oracle, SIGM algorithm converges with rate $O\left(\frac{\sigma}{\sqrt{n}} + \gamma \right)$. In order to have $O\left(\frac{\sigma}{\sqrt{n}} + \gamma \right) \leqslant \alpha$, it suffices if $n > O\left(\frac{p^{4p+2}(4r)^{2p+2}}{\alpha^2 \epsilon^{2p+4}}\right)  = O\left((\frac{8r}{\alpha})^{4r\ln\ln(8r/\alpha)} \left(\frac{4r}{\epsilon}\right)^{2cr\ln(8r/\alpha)+2}\left(\frac{1}{\alpha^2\epsilon^2}\right)\right)$, as $\sigma = O\left(\frac{p^{2p+1}(4r)^{p+1}}{\epsilon^{p+2}}\right)$ according to lemma \ref{sovariancelem} (ignoring negligible $\sigma_0, \gamma$).
\end{proof}

\end{document}